\documentclass{article}
\usepackage[utf8]{inputenc}
\usepackage{iclr2019_conference,times}
\usepackage{amsthm}
\usepackage{amssymb}
\usepackage{amsmath}
\usepackage{graphicx}

\newtheorem{theorem}{Theorem}
\newtheorem{definition}{Definition}
\newtheorem{lemma}{Lemma}

\def\dd{\mathrm{d}}
\def\vw{\mathbf{w}}

\def\va{\mathbf{a}}
\def\vx{\mathbf{x}}

\def\vv{\mathbf{v}}
\def\vg{\mathbf{g}}

\def\vc{\mathbf{c}}
\def\vf{\mathbf{f}}
\def\vp{\mathbf{p}}

\def\rr{\mathbb{R}}
\def\ii#1{\mathbb{I}\left[#1\right]}
\def\pr{\mathbb{P}}
\def\tpr{\tilde{\mathbb{P}}}

\def\todo#1{\textbf{TODO:} [#1]} 

\newcommand{\yuandong}[1]{\textcolor{red}{[Yuandong: #1]}}

\def\ee#1{\mathbb{E}\left[#1\right]}
\def\ee2#1#2{\mathbb{E}_{#1}\left[#2\right]}

\def\dd{\mathrm{d}}
\def\vzero{\mathbf{0}}
\def\vone{\mathbf{1}}

\def\spx#1{\boldsymbol{\Delta}^{#1}}

\def\reg#1{x_{#1}}
\def\regcond#1{\reg{-#1} | \reg{#1}} 
\def\fullreg{x}

\vspace{-0.1in}

\title{A theoretical framework for deep locally connected ReLU network}
\author{Yuandong Tian \\
Facebook AI Research \\
\texttt{yuandong@fb.com}}

\begin{document}
\maketitle
\vspace{-0.1in}

\begin{abstract}
Understanding theoretical properties of deep and locally connected nonlinear network, such as deep convolutional neural network (DCNN), is still a hard problem despite its empirical success. In this paper, we propose a novel theoretical framework for such networks with ReLU nonlinearity. The framework explicitly formulates data distribution, favors disentangled representations and is compatible with common regularization techniques such as Batch Norm. The framework is built upon teacher-student setting, by expanding the student forward/backward propagation onto the teacher's computational graph. The resulting model does not impose unrealistic assumptions (e.g., Gaussian inputs, independence of activation, etc). Our framework could help facilitate theoretical analysis of many practical issues, e.g. overfitting, generalization, disentangled representations in deep networks. 
\end{abstract}

\vspace{-0.1in}
\section{Introduction}
Deep Convolutional Neural Network (DCNN) has achieved a huge empirical success in multiple disciplines (e.g., computer vision~\citep{alexnet,vgg,resnet}, Computer Go~\citep{alphago,alphagozero,darkforest}, and so on). On the other hand, its theoretical properties remain an open problem and an active research topic. 

Learning deep models are often treated as non-convex optimization in a high-dimensional space. From this perspective, many properties in deep models have been analyzed: landscapes of loss functions~\citep{landscape-anna, skip-connection-landscape-better, mei2016landscape}, saddle points~\citep{exponential-time-saddle-point, yannd-saddle-point}, relationships between local minima and global minimum~\citep{kenji-local-min-global-min, hardt2016identity, DBLP:journals/corr/abs-1712-08968}, trajectories of gradient descent~\citep{goodfellow2014qualitatively}, path between local minima~\citep{venturi2018neural}, etc.

However, two components are missing: such a modeling does not consider specific network structure and input data distribution, both of which are critical factors in practice. Empirically, deep models work particular well for certain forms of data (e.g., images); theoretically, for certain data distribution, popular methods like gradient descent is shown to be unable to recover the network parameters~\citep{brutzkus2017globally}.

Along this direction, previous theoretical works assume specific data distributions like spherical Gaussian and focus on shallow nonlinear networks~\citep{tian2017analytical,brutzkus2017globally,du-spurious-local-min-icml18}. These assumptions yield nice forms of gradient to enable analysis of many properties such as global convergence, which makes it nontrivial to extend to deep nonlinear neural networks that yield strong empirical performance. 

In this paper, we propose a novel theoretical framework for deep locally connected ReLU network that is applicable to general data distributions. Specifically, we embrace a teacher-student setting: the \emph{teacher} generates classification label via a hidden computational graph, and the \emph{student} updates the weights to fit the labels with gradient descent. Then starting from gradient descent rule, we marginalize out the input data conditioned on the graph variables of the teacher at each layer, and arrive at a \emph{reformulation} that \textbf{(1)} captures data distribution as explicit terms and leads to more interpretable model, \textbf{(2)} compatible with existing state-of-the-art regularization techniques such as Batch Normalization~\citep{batchnorm}, and \textbf{(3)} favors disentangled representation when data distributions have factorizable structures. To our best knowledge, our work is the first theoretical framework to achieve these properties for deep and locally connected nonlinear networks.

Previous works have also proposed framework to explain deep networks, e.g., renormalization group for restricted Boltzmann machines~\citep{mehta2014exact}, spin-glass models~\citep{amit1985spin,choromanska2015loss}, transient chaos models~\citep{poole2016exponential}, differential equations~\citep{su2014differential, saxe2013exact}. In comparison, our framework \textbf{(1)} imposes mild assumptions rather than unrealistic ones (e.g., independence of activations), \textbf{(2)} explicitly deals with back-propagation which is the dominant approach used for training in practice, \textbf{(3)} considers spatial locality of neurons, an important component in practical deep models, and \textbf{(4)} models data distribution explicitly. 

The paper is organized as follows: Sec.~\ref{sec:basic-formulation} introduces a novel approach to model locally connected networks. Sec.~\ref{sec:teacher-student-setting} introduces the teacher-student setting and label generation, followed by the proposed reformulation. Sec.~\ref{sec:batch-norm-under-coarse-model} gives one novel finding that Batch Norm is a projection onto the orthogonal complementary space of neuron activations, and the reformulation is compatible with it. Sec.~\ref{sec:property-of-coarse-model} shows a few applications of the framework, e.g., why nonlinearity is helpful, how factorization of the data distribution leads to disentangled representation and other issues. 

\def\ch{\mathrm{ch}}
\def\pa{\mathrm{pa}}
\def\Z{\mathcal{Z}}
\def\rf#1{{\mathrm{rf}(#1)}}
\def\raw{\mathrm{raw}}
\def\sign{\mathrm{sign}}

\section{Basic formulation}
\label{sec:basic-formulation}
\subsection{General setting}
In this paper, we consider multi-layer (deep) network with ReLU nonlinearity. We consider supervised setting, in which we have a dataset $\{(x, y)\}$, where $x$ is the input image and $y$ is its label computed from $x$ in a deterministic manner. Sec.~\ref{sec:teacher-student-setting} describes how $x$ maps to $y$ in details.  

We consider a neuron (or node) $j$. Denote $f_j$ as its activation after nonlinearity and $g_j$ as the (input) gradient it receives after filtered by ReLU's gating. Note that both $f_j$ and $g_j$ are deterministic functions of the input $x$ and label $y$. Since $y$ is a deterministic function of $x$, we can write $f_j = f_j(x)$ and $g_j=g_j(x)$. Note that all analysis still holds with bias terms. We omit them for brevity.

The activation $f_j$ and gradient $g_k$ can be written as (note that $f'_j$ is the binary gating function):
\begin{equation}
    f_j(x) = f'_j(x) \sum_{k\in \ch(j)} w_{jk} f_k(x), \quad g_k(x) = f'_k(x) \sum_{j\in \pa(k)} w_{jk} g_j(x)
    \label{eq:x-update}
\end{equation}
And the weight update for gradient descent is:
\begin{equation}
    \Delta w_{jk} = \ee2{x}{f_k(x)g_j(x)}\label{eq:x-weight-update}
\end{equation}
Here is the expectation is with respect to a training dataset (or a batch), depending on whether GD or SGD has been used. We also use $f_j^\raw$ and $g_j^\raw$ as the counterpart of $f_j$ and $g_j$ before nonlinearity.

\begin{figure}
    \centering
    \includegraphics[width=\textwidth]{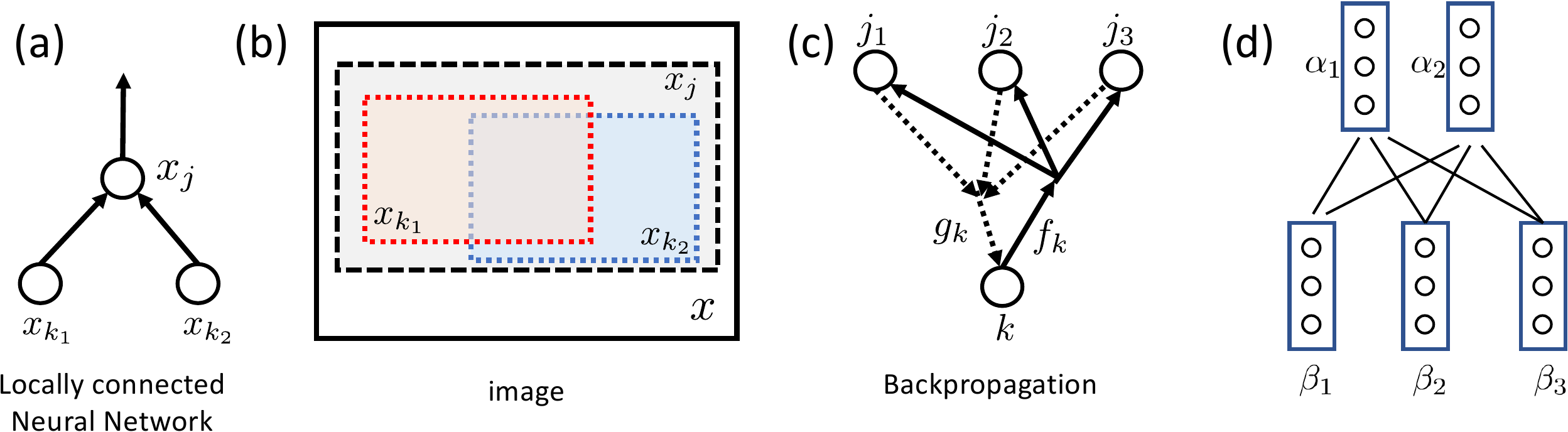}
    \caption{Problem Setting. (a) Locally connected network,  (b) the receptive fields of each node. (c) notations used in backpropagation. (d) nodes with same receptive fields are grouped (Eqn.~\ref{eqn:matrix-form-simplified}).}
    \label{fig:lcn}
\end{figure}

\subsection{Locally Connected Network}
Locally connected networks have extra structures, which leads to our reformulation. As shown in Fig.~\ref{fig:lcn}, each node $j$ only covers one part of the input images (i.e., \emph{receptive field}). We use Greek letters $\{\alpha, \beta, \ldots, \omega\}$ to represent receptive fields. For a region $\alpha$, $x_\alpha$ is the content in that region. $j\in\alpha$ means node $j$ covers the region of $\alpha$ and  $n_\alpha$ is the number of nodes that cover the same region (e.g., multi-channel case). The image content is $\reg{\alpha(j)}$, abbreviated as $\reg{j}$ if no ambiguity. The parent $j$'s receptive field covers its children's. Finally, $\omega$ represents the entire image.

By definition, the activation $f_j$ of node $j$ is only dependent on the region $\reg{j}$, rather than the entire image $x$. This means that $f_j(\fullreg) = f_j(\reg{j})$ and $f_j(\reg{j}) = f'_j(x_j)\sum_k w_{jk} f_k(\reg{k})$. However, the gradient $g_j$ is determined by the entire image $\fullreg$, and its label $y$, i.e., $g_j = g_j(\fullreg, y)$. Note that here we assume that the label $y$ is a deterministic (but unknown) function of $\fullreg$. Therefore, for gradient we just write $g_j = g_j(\fullreg)$.

\subsection{Marginalized Gradient} 
Given the structure of locally connected network, the gradient $g_j$ has some nice structures. From Eqn.~\ref{eq:x-weight-update} we knows that
$\Delta w_{jk} = \ee2{x}{f_k(\fullreg)g_j(\fullreg)} = \ee2{\reg{k}}{f_k(\reg{k}) \ee2{x_{-k} | x_k}{g_j(x)}}$. Define $\reg{-k} = \fullreg \backslash \reg{k}$ as the input image $x$ except for $\reg{k}$. Then we can define the \emph{marginalized gradient}:
\begin{equation}
g_j(x_k) = \ee2{\regcond{k}}{g_j(\fullreg)}
\end{equation}
as the marginalization (average) of $\reg{-k}$, while keep $\reg{k}$ fixed. With this notation, we can write $\Delta w_{jk} = \ee2{\reg{k}}{f_k(\reg{k})g_j(\reg{k})}$.

On the other hand, the gradient which  back-propagates to a node $k$ can be written as
\begin{equation}
    g_k(\fullreg) = f'_k(\fullreg) \sum_{j \in \pa(k)} w_{jk} g_j(\fullreg) = f'_k(\reg{k}) \sum_j w_{jk} g_j(\fullreg)
\end{equation}
where $f'_k$ is the derivative of activation function of node $k$ (for ReLU it is just a gating function). If we take expectation with respect to $\regcond{k}$ on both side, we get 
\begin{equation}
    g_k(\reg{k}) = f'_k(\reg{k})g^{\mathrm{raw}}_k(\reg{k}) = f'_k(\reg{k}) \sum_{j\in \pa(k)} w_{jk} g_j(\reg{k}) \label{eq:grad-collect}
\end{equation}
Note that all marginalized gradients $g_j(\reg{k})$ are independently computed by marginalizing with respect to all regions that are outside the receptive field $\reg{k}$. Interestingly, there is a relationship between these gradients that respects the locality structure:
\begin{theorem}[Recursive Property of marginalized gradient]
$g_j(\reg{k}) = \ee2{\reg{j, -k} | \reg{k}}{g_j(\reg{j})}$ \end{theorem}

This shows that the marginal gradient has recursive structure: we can first compute $g_j(\reg{j})$ for top node $j$, then by marginalizing over the region within $\reg{j}$ but outside $\reg{k}$, we get its projection $g_j(\reg{k})$ on child $k$, then by Eqn.~\ref{eq:grad-collect} we collect all projections from all the parents of node $k$, to get $g_k(\reg{k})$. This procedure can be repeated until we arrive at the leaf nodes. 

\section{Teacher-Student Setting}
\label{sec:teacher-student-setting}
In order to analyze the behavior of neural network under backpropagation (BP), one needs to make assumptions about how the input $x$ is generated and how the label $y$ is related to the input $x$. Previous works assume Gaussian inputs and shallow networks, which yields analytic solution to gradient~\citep{tian2017analytical,du-spurious-local-min-icml18}, but might not align well with the practical data distribution.

\subsection{The Summarization Ladder}
We consider a multi-layered deterministic function as the teacher (not necessary a neural network). For each region $x_\alpha$, there is a latent discrete \emph{summarization} variable $z_\alpha$ that captures the information of the input $x_\alpha$. Furthermore, we assume $z_\alpha = z_\alpha(x_\alpha) = z_\alpha(\{z_\beta\}_{\beta\in\ch(\alpha)})$, i.e., the summarizaion only relies on the ones in the immediate lower layer. In particular, the top level summarization is the \emph{label} of the image, $y = z_\omega$, where $\omega$ represents the region of the entire image. We call a particular assignment of $z_\alpha$, $z_\alpha = a$, an \emph{event}. Finally, $m_\alpha$ is how many values $z_\alpha$ can take.

During training, all summarization functions $\Z = \{z_\alpha\}$ are unknown except for the label $y$. 

\subsection{Function Expansion on Summarization}
Let's consider the following quantity. For each neural node $j$, we want to compute the expected gradient given a particular \emph{factor} $z_\alpha$, where $\alpha=\rf{j}$ (the reception field of node $j$):
\begin{equation}
    g_j(z_\alpha) \equiv \ee2{X_j|z_\alpha}{g_j(X_j)} = \int g_j(x_j) \pr(x_j|z_\alpha) \dd x_j
\end{equation}
And $\tilde g_j(z_\alpha) =  g_j(z_\alpha)\pr(z_\alpha)$. Similarly, $f_j(z_\alpha) = \ee2{X_j|z_\alpha}{f_j(X_j)}$ and $f'_j(z_\alpha) = \ee2{X_j|z_\alpha}{f'_j(X_j)}$. 

Note that $\pr(x_j | z_\alpha)$ is the \emph{frequency count} of $x_j$ for $z_\alpha$. If $z_\alpha$ captures all information of $x_j$, then $\pr(x_j | z_\alpha)$ is a delta function. Throughout the paper, we use frequentist interpretation of probabilities. 

Intuitively, if we have $g_j(z_\alpha = a) > 0$ and $g_j(z_\alpha \neq a) < 0$, then the node $j$ learns about the \emph{hidden} event $z_\alpha = a$. For multi-class classification, the top level nodes (just below the softmax layer) already embrace such correlations (here $j$ is the class label):
\begin{equation}
    g_j(y = j) > 0,\quad g_j(y \neq j) < 0, \label{eqn:grad-top-level}
\end{equation}
where we know $z_\omega = y$ is the top level factor. A natural question now arises: 

\begin{center}
\emph{Does gradient descent automatically push $g_j(z_\alpha)$ to be correlated with the factor $z_\alpha$}?     
\end{center}

If this is true, then 
gradient descent on deep models is essentially a \emph{weak-supervised} approach that automatically learns the intermediate events at different levels. Giving a complete answer of this question is very difficult and is beyond the scope of this paper. Here, we aim to build a theoretical framework that enables such analysis. We start with the relationship between neighboring layers:

\begin{theorem}[Reformulation]
\label{thm:coarse-model}
Denote $\alpha = \rf{j}$ and $\beta=\rf{k}$. $k$ is a child of $j$. If the following conditions hold:
\begin{itemize}
    \item \textbf{Focus of knowledge}. $\pr(x_k|z_\alpha, z_\beta) = \pr(x_k|z_\beta)$. 
    \item \textbf{Broadness of knowledge}. $\pr(x_j|z_\alpha, z_\beta) = \pr(x_j|z_\alpha)$. 
    \item \textbf{Decorrelation}. Given $z_\beta$, ($g_k^{\mathrm{raw}}(\cdot)$ and $f'_k(\cdot)$) and ($f_k^{\mathrm{raw}}(\cdot)$ and $f'_k(\cdot)$) are \emph{uncorrelated}.
\end{itemize}
Then the following iterative equations hold:
\begin{eqnarray}
f_j(z_\alpha)\!= \! f_j'(z_\alpha)\!\sum_{k\in\ch(j)}w_{jk}\ee2{z_\beta|z_\alpha}{f_k(z_\beta)},\quad
g_k(z_\beta)\!=\! f_k'(z_\beta)\!\sum_{j\in\pa(k)}w_{jk}\ee2{z_\alpha|z_\beta}{g_j(z_\alpha)} \label{eq:induction}
\end{eqnarray}
\end{theorem}
One key property of this formulation is that, it incorporates data distributions $\pr(z_\alpha, z_\beta)$ into the gradient descent rules. This is important since running BP on different dataset is now formulated into the same framework with different probability, i.e., frequency counts of events. By studying which family of distribution leads to the desired property, we could understand BP better. 

For completeness, we also need to define boundary conditions. In the lowest level $L$, we could treat each input pixel (or a group of pixels) as a single event. Therefore, $f_k(z_\beta) = \ii{k = z_\beta}$. On the top level, as we have discussed, Eqn.~\ref{eqn:grad-top-level} applies and $g_j(z_\beta) = a_1\ii{j = z_\beta} - a_2\ii{j \neq z_\beta}$. 

The following theorem shows that the reformulation is exact if $z_\alpha$ has all information of the region. 
\begin{theorem}
If $\pr(x_j|z_\alpha)$ is a delta function for all $\alpha$, then all conditions in Thm.~\ref{thm:coarse-model} hold.
\end{theorem}

In general, $\pr(x_j|z_\alpha)$ is a distribution encoding how much information gets lost if we only know the factor $z_\alpha$. When we climb up the ladder, we lose more and more information while keeping the critical part for the classification. 
This is consistent with empirical observations~\citep{bau2017network}, in which the low-level features in DCNN are generic, and high-level features are more class-specific. 

\subsection{Matrix Formulation}
Eqn.~\ref{eq:induction} can be hard to deal with. If we group the nodes with the same reception field at the same level together (Fig.~\ref{fig:lcn}(d)), we have the matrix form ($\circ$ is element-wise multiplication):
\begin{table}
\centering
\begin{tabular}{|l||l|l|}
\hline
 & Dimension & Description \\
\hline\hline
$F_\alpha$, $\tilde G_\alpha$, $D_\alpha$ & $m_\alpha$-by-$n_\alpha$ & Activation $f_j(z_\alpha)$, gradient $\tilde g_j(z_\alpha)$ and gating prob $f'_j(z_\alpha)$ at group $\alpha$. \\
\hline
$W_{\beta\alpha}$ & $n_\beta$-by-$n_\alpha$ & Weight matrix that links group $\alpha$ and $\beta$ \\
\hline
$P_{\alpha\beta}$ & $m_\alpha$-by-$m_\beta$ & Prob $\pr(z_\beta|z_\alpha)$ of events at group $\alpha$ and $\beta$ \\
\hline
\end{tabular}
\caption{Matrix Notation. See  Eqn.~\ref{eqn:matrix-form-simplified}.}
\label{tbl:matrix-notation}
\end{table}

\begin{theorem}[Matrix Representation of Reformulation]
\begin{equation}
    F_\alpha\! =\!D_{\alpha} \!\circ\! \sum_{\beta\in \ch(\alpha)} P_{\alpha\beta} F_\beta W_{\beta\alpha} ,\quad
    \tilde G_\beta\! =\! D_\beta \!\circ\! \sum_{\alpha\in\pa(\beta)} P_{\alpha\beta}^T \tilde G_\alpha W_{\beta\alpha}^T  ,\quad \Delta W_{\beta\alpha}\! =\! (P_{\alpha\beta} F_\beta)^T \tilde G_\alpha
    \label{eqn:matrix-form-simplified}  
\end{equation}
\end{theorem}
See Tbl.~\ref{tbl:matrix-notation} for the notation. For this dynamics, we want $F^*_\omega = I_{n_\omega}$, i.e., the top $n_\omega$ neurons faithfully represents the classification labels. Therefore, the top level gradient is $G_\omega = I_{n_\omega} - F_\omega$. On the other side, for each region $\beta$ at the bottom layer, we have $F_\beta = I_{n_\beta}$, i.e., the input contains all the preliminary factors. For all regions $\alpha$ in the top-most and bottom-most layers, we have $n_\alpha=m_\alpha$.

\def\iter#1{{(#1)}}
\def\vert{\mathrm{vert}}
\def\CH{\mathrm{Conv}}
\def\rank{\mathrm{rank}}
\def\bninput#1{f^{(#1)}}
\def\bnzeromean#1{\hat f^{(#1)}}
\def\bnstandard#1{\tilde f^{(#1)}}
\def\bnoutput#1{\bar f^{(#1)}}

\def\bnvinput{\vf}
\def\bnvzeromean{\hat \vf}
\def\bnvstandard{\tilde \vf}
\def\bnvoutput{\bar \vf}

\section{Batch Normalization under Reformulation}
\label{sec:batch-norm-under-coarse-model}
Our reformulation naturally incorporates empirical regularization technique like Batch Normalization (BN)~\citep{batchnorm}. 

\begin{figure}
    \centering
    \includegraphics[width=0.8\textwidth]{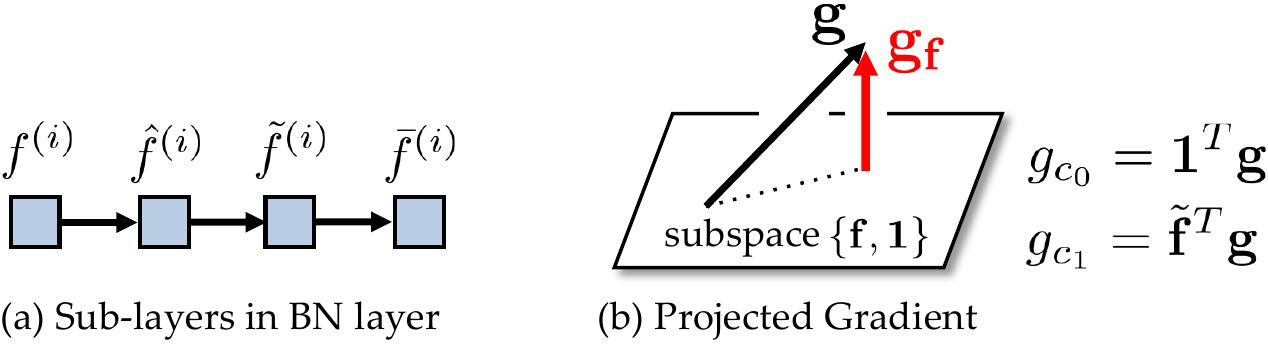}
    \caption{Batch Normalization (BN) as a projection. \textbf{(a)} Three sublayers in BN (zero-mean, unit-variance, affine). \textbf{(b)} The 
      gradient $\vg_\vf$ that is propagated down is a projection of input gradient $\vg$ onto the orthogonal complementary space spanned by $\{\vf, \vone\}$.}\label{fig:bn}
\end{figure}

\subsection{Batch Normalization as a Projection}
We start with a novel finding of Batch Norm: the back-propagated gradient through Batch Norm layer at a node $j$ is a projection onto the orthogonal complementary subspace spanned by all one vectors and the current activations of node $j$. 

Denote pre-batchnorm activations as $\bninput{i} = f_j(x_i)$ ($i = 1 \ldots N$). In Batch Norm, $\bninput{i}$ is whitened to be $\bnstandard{i}$, then linearly transformed to yield the output $\bnoutput{i}$:
\begin{equation}
    \bnzeromean{i} = \bninput{i} - \mu, \quad\bnstandard{i} = \bnzeromean{i} /\sigma, \quad\bnoutput{i} = c_1 \bnstandard{i} + c_0
\end{equation}
where $\mu = \frac{1}{N}\sum_i \bninput{i}$ and $\sigma^2 = \frac{1}{N}\sum_i (\bninput{i} - \mu)^2$ and $c_1$, $c_0$ are learnable parameters.

The original Batch Norm paper derives complicated and unintuitive weight update rules. With vector notation, the update has a compact form with clear geometric meaning. 

\begin{theorem}[Backpropagation of Batch Norm]
\label{thm:bn}
For a top-down gradient $\vg$, BN layer gives the following gradient update ($P^\perp_{\bnvinput, \vone}$ is the orthogonal complementary projection of subspace $\{\bnvinput, \vone\}$):
\begin{equation}
    \vg_\vf = J^{BN}(\bnvinput)\vg = \frac{c_1}{\sigma}P^\perp_{\bnvinput, \vone}\vg, \quad \vg_\vc = S(\bnvinput{})^T \vg
    \label{eq:batch-norm-projection}
\end{equation}
\end{theorem}
Intuitively, the back-propagated gradient $J^{BN}(\bnvinput)\vg$ is zero-mean and perpendicular to the input activation $\bnvinput$ of BN layer, as illustrated in Fig.~\ref{fig:bn}. Unlike~\citep{kohler2018towards} that analyzes BN in an approximate manner, in Thm.~\ref{thm:bn} we do not impose any assumptions. 

\subsection{Batch Norm under the reformulation}
The analysis of Batch Norm is compatible with the reformulation and we arrive at similar backpropagation rule, by noticing that $\ee2{x}{f_j(x)} = \ee2{z_\alpha}{f_j(z_\alpha)}$: 
\begin{equation}
\mu = \ee2{z_\alpha}{f_j},\quad \sigma^2 = \ee2{z_\alpha}{(f_j(z_\alpha) - \mu)^2}, \quad J^{BN}(\bnvinput{}) = \frac{c_1}{\sigma}P^\perp_{\bnvinput, \vone} \label{eq:batch-norm-projection-z-alpha}
\end{equation}
Note that we still have the projection property, but under the new inner product $\langle f_j, g_j\rangle_{z_\alpha} = \ee2{z_\alpha}{f_j(z_\alpha)g_j(z_\alpha)}$ and norm $\|f\|_{z_\alpha} = \langle f, f\rangle_{z_\alpha}^{1/2}$.  

\def\sz{\mathrm{sz}}

\section{Example applications of proposed theoretical framework}
\label{sec:property-of-coarse-model}
With the help of the theoretical framework, we now can analyze interesting structures of gradient descent in deep models, when the data distribution $\pr(z_\alpha, z_\beta)$ satisfies specific conditions. Here we give two concrete examples: the role played by nonlinearity and in which condition disentangled representation can be achieved. Besides, from the theoretical framework, we also give general comments on multiple issues (e.g., overfitting, GD versus SGD) in deep learning. 

\subsection{Nonlinear versus linear}
\label{sec:nonlinear-vs-linear}
In the formulation, $m_\alpha$ is the number of possible events within a region $\alpha$, which is often exponential with respect to the size $\sz(\alpha)$ of the region. The following analysis shows that a linear model cannot handle it, even with exponential number of nodes $n_\alpha$, while a nonlinear one with ReLU can. 

\begin{definition}[Convex Hull of a Set]
We define the convex hull $\CH(P)$ of $m$ points $P \subset \rr^n$ to be $\CH(P) = \left\{P\va, \va\in\spx{n-1}\right\}$, where $\spx{n-1} = \left\{\va\in \rr^n, a_i \ge 0, \sum_i a_i = 1\right\}$. A row $p_j$ is called \emph{vertex} if $p_j \notin \CH(P \backslash p_j)$.  
\end{definition}

\begin{definition}
A matrix $P$ of size $m$-by-$n$ is called \emph{$k$-vert}, or $\vert(P) = k \le m$, if its $k$ rows are vertices of the convex hull generated by its rows. $P$ is called \emph{all-vert} if $k = m$.
\end{definition}

\begin{theorem}[Expressibility of ReLU Nonlinearity]
\label{thm:sufficient-node}
Assuming $m_\alpha = n_\alpha = \mathcal{O}(\exp(\mathrm{sz}(\alpha)))$, where $\sz(\alpha)$ is the size of receptive field of $\alpha$. If each $P_{\alpha\beta}$ is all-vert, then: ($\omega$ is top-level receptive field) 
\begin{equation}
    \min_W Loss_\mathrm{ReLU}(W) = 0, \quad \min_W Loss_\mathrm{Linear}(W) = \mathcal{O}(\exp(\sz(\omega)))
\end{equation}
\end{theorem}
Note that here $Loss(W) \equiv \|F_\omega - I\|^2_F$. This shows the power of nonlinearity, which guarantees full rank of output, even if the matrices involved in the multiplication are low-rank. The following theorem shows that for intermediate layers whose input is not identity, the all-vert property remains. 
\begin{theorem}
    \textbf{\emph{(1)}} If $F$ is full row rank, then $\vert(PF) = \vert(P)$. \textbf{\emph{(2)}} $PF$ is all-vert iff $P$ is all-vert.
\end{theorem}
This means that if all $P_{\alpha\beta}$ are all-vert and its input $F_\beta$ is full-rank, then with the same construction of Thm.~\ref{thm:sufficient-node}, $F_\alpha$ can be made identity. In particular, if we sample $W$ randomly, then with probability $1$, all $F_\beta$ are full-rank, in particular the top-level input $F_1$. Therefore, using top-level $W_1$ alone would be sufficient to yield zero generalization error, as shown in the previous works that random projection could work well. 

\def\bfactor#1#2{#1[#2]}

\begin{figure}
    \centering
    \includegraphics[width=\textwidth]{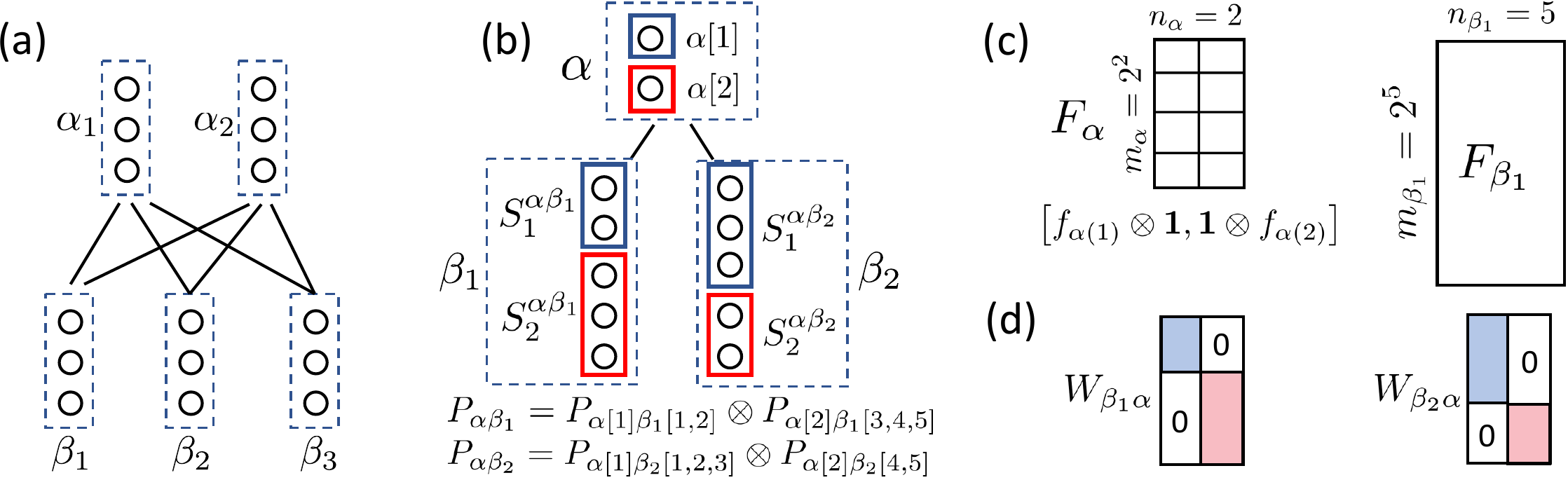}
    \caption{disentangled representation. \textbf{(a)} Nodes are grouped according to regions. \textbf{(b)} An example of one parent region $\alpha$ ($2$ nodes) and two child regions $\beta_1$ and $\beta_2$ ($5$ nodes each). We assume factorization property of data distribution $P$. \textbf{(c)} disentangled activations, \textbf{(d)} Separable weights.}
    \label{fig:disentangled-reprsentation}
\end{figure}

\subsection{disentangled Representation}
The analysis in Sec.~\ref{sec:nonlinear-vs-linear} assumes that $n_\alpha = m_\alpha$, which means that we have sufficient nodes, \emph{one neuron for one event}, to convey the information forward to the classification level. In practice, this is never the case. When $n_\alpha \ll m_\alpha = \mathcal{O}(\exp(\mathrm{sz}(\alpha)))$ and the network needs to represent the information in a proper way so that it can be sent to the top level. Ideally, if the factor $z_\alpha$ can be written down as a list of binary factors: $z_\alpha = \left[z_{\bfactor\alpha1}, z_{\bfactor\alpha2}, \ldots, z_{\bfactor\alpha{j}}\right]$, the output of a node $j$ could represent $z_{\bfactor{\alpha}{j}}$, so that all $m_\alpha$ events can be represented concisely with $n_\alpha$ nodes. 

To come up with a complete theory for disentangled representation in deep nonlinear network is far from trivial and beyond the scope of this paper. In the following, we make an initial attempt by constructing factorizable $P_{\alpha\beta}$ so that disentangled representation is possible in the forward pass. First we need to formally define what is disentangled representation:
\begin{definition}
The activation $F_\alpha$ is \emph{disentangled}, if its $j$-th column $F_{\alpha,:j} = \vone \otimes \ldots \!\otimes\! \vf_{\bfactor{\alpha}{j}} \!\otimes\! \ldots\! \otimes\! \vone $, where each $\vf_{\bfactor{\alpha}{j}}$ and $\vone$ is a $2$-by-$1$ vector.
\end{definition}
\begin{definition}
The gradient $\tilde G_\alpha$ is \emph{disentangled}, if its $j$-th column $\tilde G_{\alpha,:j} = \vp_{\bfactor{\alpha}{1}} \otimes \ldots \!\otimes\!\tilde\vg_{\bfactor{\alpha}{j}} \!\otimes\! \ldots\! \otimes\! \vp_{\bfactor{\alpha}{n_\alpha}} $, where $\vp_{\bfactor{\alpha}{j}} = \left[\pr(\bfactor{\alpha}{j} = 0), \pr(\bfactor{\alpha}{j} = 1)\right]^T$ and $\tilde \vg_{\bfactor{\alpha}{j}}$ is a $2$-by-$1$ vector.
\end{definition}
Intuitively, this means that each node $j$ represents the binary factor $z_\bfactor{\alpha}{j}$. A follow-up question is whether such disentangled properties carries over layers in the forward pass. It turns out that the disentangled structure carries if the data distribution and weights have compatible structures:  

\begin{definition}
The weights $W_{\beta\alpha}$ is \emph{separable} with respect to a disjoint set $\{S^{\alpha\beta}_i\}$, if $W_{\beta\alpha} = \mathrm{diag}\left(W_{\beta\alpha}[S^{\alpha\beta}_1, 1], W_{\beta\alpha}[S^{\alpha\beta}_2, 2], \ldots, W_{\beta\alpha}[S^{\alpha\beta}_{n_\alpha}, n_\alpha]\right)$.
\end{definition}

\begin{theorem}[Disentangled Forward]
If for each $\beta\in\ch(\alpha)$,  $P_{\alpha\beta}$ can be written as a tensor product $P_{\alpha\beta} = \bigotimes_i P_{\bfactor{\alpha}{i}\bfactor{\beta}{S^{\alpha\beta}_i}}$ where $\{S^{\alpha\beta}_i\}$ are $\alpha\beta$-dependent disjointed set, $W_{\beta\alpha}$ is separable with respect to $\{S^{\alpha\beta}_i\}$, $F_\beta$ is disentangled, then $F_\alpha$ is also disentangled (with/without ReLU /Batch Norm).
\end{theorem}

If the bottom activations are disentangled, by induction, all activations will be disentangled. The next question is whether gradient descent preserves such a structure. The answer is also conditionally yes:

\begin{theorem}[Separable Weight Update]
If $P_{\alpha\beta} = \bigotimes_i P_{\bfactor{\alpha}{i}\bfactor{\beta}{S_i}}$, $F_\beta$ and $\tilde G_\alpha$ are both disentangled, $\vone^T \tilde G_\alpha = \vzero$, then the gradient update $\Delta W_{\beta\alpha}$ is separable with respect to $\{S_i\}$.
\end{theorem}
Therefore, with disentangled $F_\beta$ and $\tilde G_\alpha$ and centered gradient $\vone^T \tilde G_\alpha = \vzero$, the separable structure is conserved over gradient descent, given the initial $W^{(0)}_{\beta\alpha}$ is separable. Note that centered gradient is guaranteed if we insert Batch Norm (Eqn.~\ref{eq:batch-norm-projection-z-alpha}) after linear layers. And the activation $F$ remains disentangled if the weights are separable. 

The hard part is whether $\tilde G_\beta$ remains disentangled during backpropagation, if $\{\tilde G_\alpha\}_{\alpha\in\pa(\beta)}$ are all disentangled. If so, then the disentangled representation is self-sustainable under gradient descent. This is a non-trivial problem and generally requires structures of data distribution. We put some discussion in the Appendix and leave this topic for future work.

\subsection{Explanation of common behaviors in Deep Learning}
In the proposed formulation, the input $x$ in Eqn.~\ref{eq:induction} is integrated out, and the data distribution is now encoded into the probabilistic distribution $\pr(z_\alpha, z_\beta)$, and their marginals. A change of such distribution means the input distribution has changed. For the first time, we can now analyze many practical factors and behaviors in the DL training that is traditionally not included in the formulation.

\textbf{Over-fitting.} Given finite number of training samples, there is always error in estimated factor-factor distribution $\tpr(z_\alpha, z_\beta)$ and factor-observation distribution $\tpr(x_\alpha|z_\alpha)$. In some cases, a slight change of distribution would drastically change the optimal weights for prediction, which is overfitting. 

Here is one example. Suppose there are two different kinds of events at two disjoint reception fields: $z_\alpha$ and $z_\gamma$. The class label is $z_\omega$, which equals $z_\alpha$ but is not related to  $z_\gamma$. Therefore, we have:
\begin{equation}
    \tpr(z_\omega = 1|z_\alpha=1) = 1,\quad \tpr(z_\omega = 1|z_\alpha=0) = 0
\end{equation}
Although $z_\gamma$ is unrelated to the class label $z_\omega$, with finite samples $z_\gamma$ could show spurious correlation: 
\begin{equation}
    \tpr(z_\omega = 1|z_\gamma=1) = 0.5 + \epsilon,\quad \tpr(z_\omega = 1|z_\gamma=0) = 0.5-\epsilon
\end{equation}

\begin{figure}
    \centering
    \includegraphics[width=0.9\textwidth]{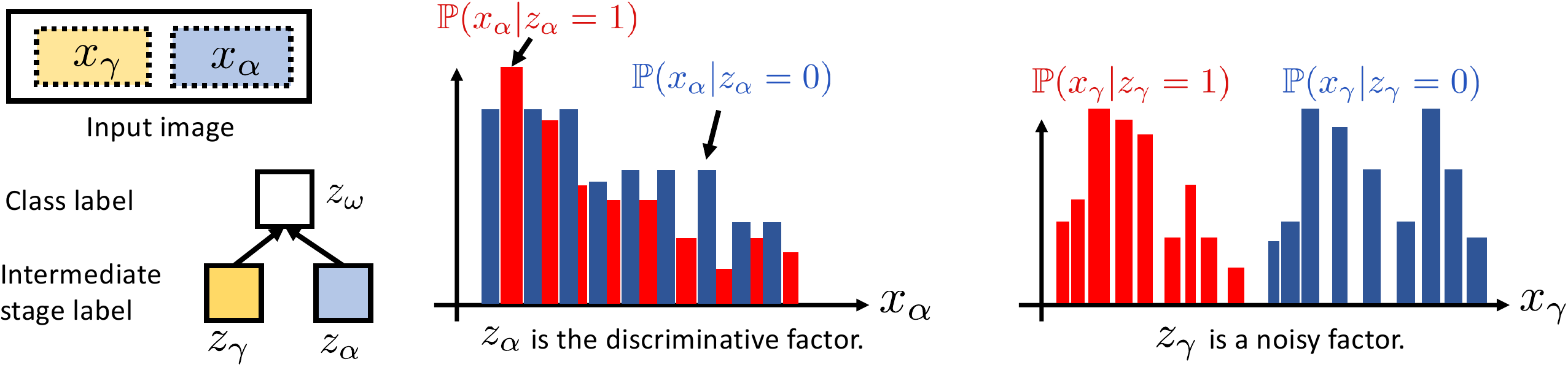}
    \caption{Overfitting Example}
    \label{fig:overfitting}
\end{figure}

On the other hand, as shown in Fig.~\ref{fig:overfitting}, $\pr(x_\alpha|z_\alpha)$ contains a lot of detailed structures and is almost impossible to separate in the finite sample case, while $\pr(x_\gamma|z_\gamma)$ could be well separated for $z_\gamma = 0/1$. Therefore, for node $j$ with $\rf{j} = \alpha$, $f_j(z_\alpha) \approx \mathrm{constant}$ (input almost indistinguishable): 
\begin{equation}
    \Delta w_j = \ee2{z_\alpha}{f_j(z_\alpha)g_0(z_\alpha)} \approx 0
\end{equation}
where $g_0(z_\alpha) = \ee2{z_\omega|z_\alpha}{g_0(z_\omega)} = \left\{ \begin{array}{cc}
1 & z_\alpha = 1 \\
-1 & z_\alpha = 0
\end{array}\right.$, which is a strong gradient signal backpropagated from the top softmax level, since $z_\alpha$ is strongly correlated with $z_\omega$. For node $k$ with $\rf{k} = \gamma$, an easy separation of the input (e.g., random initialization) yields distinctive $f_k(z_\gamma)$. Therefore, 
\begin{equation}
    \Delta w_k = \ee2{z_\gamma}{f_j(z_\gamma)g_0(z_\gamma)} > 0
\end{equation}
where $g_0(z_\gamma) = \ee2{z_\omega|z_\gamma}{g_0(z_\omega)} = \left\{ \begin{array}{cc}
2\epsilon & z_\gamma = 1 \\
-2\epsilon & z_\gamma = 0
\end{array}\right.$, a weak signal because of $z_\gamma$ is (almost) unrelated to the label. Therefore, we see that the weight $w_j$ that links to meaningful receptive field $z_\alpha$ does not receive strong gradient, while the weight $w_k$ that links to irrelevant (but spurious) receptive field $z_\gamma$ receives strong gradient. This will lead to overfitting. 

With more data, over-fitting is alleviated since (1) $\tpr(z_\omega|z_\gamma)$ becomes more accurate and $\epsilon \rightarrow 0$; (2) $\tpr(x_\alpha|z_\alpha)$ starts to show statistical difference for $z_\alpha=0/1$ and thus $f_j(z_\alpha)$ shows distinctiveness.

Note that there exists a \textbf{second} explanation: we could argue that $z_\gamma$ is a \emph{true} but \emph{weak} factor that contributes to the label, while $z_\alpha$ is a \emph{fictitious} discriminative factor, since the appearance difference between $z_\alpha=0$ and $z_\alpha=1$ (i.e., $\tpr(x_\alpha|z_\alpha)$ for $\alpha=0/1$) could be purely due to noise and thus should be neglected. With finite number of samples, these two cases are essentially indistinguishable. Models with different induction bias might prefer one to the other, yielding drastically different generalization error. For neural network, SGD prefers the second explanation but if under the pressure of training, it may also explore the first one by pushing gradient down to distinguish subtle difference in the input. This may explain why the same neural networks can fit random-labeled data, and generalize well for real data~\citep{zhang2016understanding}.

\textbf{Gradient Descent: Stochastic or not?} Previous works~\citep{keskar2016large} show that empirically stochastic gradient decent (SGD) with small batch size tends to converge to ``flat'' minima and offers better generalizable solution than those uses larger batches to compute the gradient.

From our framework, SGD update with small batch size is equivalent to using a perturbed/noisy version of $\pr(z_\alpha, z_\beta)$ at each iteration. Such an approach naturally reduces aforementioned over-fitting issues, which is due to hyper-sensitivity of data distribution and makes the final weight solution invariant to changes in $\pr(z_\alpha, z_\beta)$, yielding a ``flat'' solution.

\section{Conclusion and future work}
In this paper, we propose a novel theoretical framework for deep (multi-layered) nonlinear network with ReLU activation and local receptive fields. The framework utilizes the specific structure of neural networks, and formulates input data distributions explicitly. Compared to modeling deep models as non-convex problems, our framework reveals more structures of the network; compared to recent works that also take data distribution into considerations, our theoretical framework can model deep networks without imposing idealistic analytic distribution of data like Gaussian inputs or independent activations. Besides, we also analyze regularization techniques like Batch Norm, depicts its underlying geometrical intuition, and shows that BN is compatible with our framework.

Using this novel framework, we have made an initial attempt to analyze many important and practical issues in deep models, and provides a novel perspective on overfitting, generalization, disentangled representation, etc. We emphasize that in this work, we barely touch the surface of these core issues in deep learning. As a future work, we aim to explore them in a deeper and more thorough manner, by using the powerful theoretical framework proposed in this paper.

\bibliography{references}
\bibliographystyle{iclr2019_conference}

\clearpage

\def\proj{\mathrm{proj}}
\def\norm#1{\|#1\|}
\def\jbn#1{{J^{BN}(#1)}}
\def\s#1{S(#1)}

\def\hvx{\hat \vx}
\def\tvx{\tilde \vx}
\def\tvone{\tilde \vone}
\def\g#1{g_{#1}}
\def\p#1#2{\frac{\partial #1}{\partial #2}}

\section{Appendix}
\setcounter{theorem}{0}
\setcounter{definition}{0}
\setcounter{lemma}{0}

\begin{figure}[ht!]
    \centering
    \includegraphics[width=0.3\textwidth]{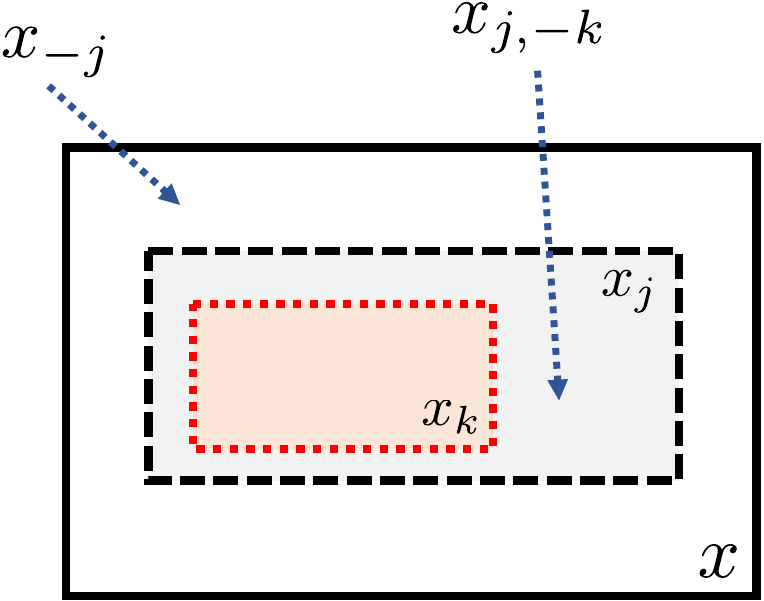}
    \caption{Notation used in Thm.~\ref{thm:hierarchical-conditioning}.}
\end{figure}

\subsection{Hierarchical Conditioning}
\begin{theorem}[Recursive Property of marginalized gradient]
\label{thm:hierarchical-conditioning}
\begin{equation}
    g_j(\reg{k}) = \ee2{\reg{j, -k} | \reg{k}}{g_j(\reg{j})} \label{eq:recursive}
\end{equation}
\end{theorem}
\begin{proof}
We have:
\begin{align*}
    g_j(\reg{k}) &= \ee2{\reg{-k} | \reg{k}}{g_j(x)} \\ 
    &= \ee2{\reg{-j}, \reg{j,-k}| \reg{k}}{g_j(x)} \\
    &= \ee2{\reg{-j}|\reg{j, -k}, \reg{k}}{\ee2{\reg{j, -k} | \reg{k}}{g_j(x)}} \\
    &= \ee2{\reg{-j}|\reg{j}}{\ee2{\reg{j, -k} | \reg{k}}{g_j(x)}} \\
    &= \ee2{\reg{j, -k} | \reg{k}}{\ee2{\reg{-j}|\reg{j}}{g_j(x)}} \\
    &= \ee2{\reg{j, -k} | \reg{k}}{g_j(\reg{j})}
\end{align*}
\end{proof}

\subsection{Network theorem}
\begin{theorem}[Reformulation]
\label{thm:coarse-model}
Denote $\alpha = \rf{j}$ and $\beta=\rf{k}$. $k$ is a child of $j$. If the following two conditions hold:
\begin{itemize}
    \item \textbf{Focus of knowledge}. $\pr(x_k|z_\alpha, z_\beta) = \pr(x_k|z_\beta)$. 
    \item \textbf{Broadness of knowledge}. $\pr(x_j|z_\alpha, z_\beta) = \pr(x_j|z_\alpha)$. 
    \item \textbf{Decorrelation}. Given $z_\beta$, ($g_k^{\mathrm{raw}}(\cdot)$ and $f'_k(\cdot)$) and ($f_k^{\mathrm{raw}}(\cdot)$ and $f'_k(\cdot)$) are \emph{uncorrelated}
\end{itemize}
Then the following two conditions holds:
\begin{subequations}
\label{eq:induction-appendix}
\begin{align}
f_j(z_\alpha) &= f_j'(z_\alpha)\sum_{k\in\ch(j)}w_{jk}\ee2{z_\beta|z_\alpha}{f_k(z_\beta)} \label{eq:induction-appendix-forward}\\
g_k(z_\beta) &= f_k'(z_\beta)\sum_{j\in\pa(k)}w_{jk}\ee2{z_\alpha|z_\beta}{g_j(z_\alpha)} \label{eq:induction-appendix-backward}
\end{align}
\end{subequations}

\end{theorem}
\begin{proof}
For Eqn.~\ref{eq:induction-appendix-forward}, we have: 
\begin{align}
    f_j^\raw(z_\alpha) &= \int f_j^\raw(x)\pr(x|z_\alpha) \dd x \\
    &= \int f_j^\raw(x_j)\pr(x_j|z_\alpha) \dd x_j \\
    &= \int \left(\sum_{k\in\ch(j)} w_{jk} f_k(x_k)\right) \pr(x_j|z_\alpha) \dd x_j \label{eq:forward2}
\end{align}
And for each of the entry, we have:
\begin{align}
    \int f_k(x_k)\pr(x_j|z_\alpha) \dd x_j &= \int f_k(x_k)\pr(x_k|z_\alpha) \dd x_k \label{eq:forward1}
\end{align}
For $\pr(x_k|z_\alpha)$, using focus of knowledge, we have:
\begin{align}
    \pr(x_k|z_\alpha) &= \sum_{z_\beta} \pr(x_k, z_\beta | z_\alpha) \\
    &= \sum_{z_\beta} \pr(x_k| z_\beta ,z_\alpha)\pr(z_\beta|z_\alpha) \\
    &= \sum_{z_\beta} \pr(x_k| z_\beta)\pr(z_\beta|z_\alpha)
\end{align}
Therefore, following Eqn.~\ref{eq:forward1}, we have:
\begin{align}
    \int f_k(x_k)\pr(x_k|z_\alpha) \dd x_k &= 
    \int f_k(x_k) \sum_{z_\beta} \pr(x_k| z_\beta)\pr(z_\beta|z_\alpha) \dd x_k \\
    &=
    \sum_{z_\beta} \left(\int f_k(x_k) \pr(x_k| z_\beta) \dd x_k\right)\pr(z_\beta|z_\alpha) \\
    &=
    \sum_{z_\beta} f_k(z_\beta) \pr(z_\beta|z_\alpha) \\
    &= \ee2{z_\beta|z_\alpha}{f_k(z_\beta)}
\end{align}
Putting it back to Eqn.~\ref{eq:forward2} and we have:
\begin{equation}
    f_j^\raw(z_\alpha) = \sum_{k\in\ch(j)} w_{jk}\ee2{z_\beta|z_\alpha}{f_k(z_\beta)} 
\end{equation}

For Eqn.~\ref{eq:induction-appendix-backward}, similarly we have:
\begin{align}
    g^{\mathrm{raw}}_k(z_\beta) &= \int g_k^\raw(x)\pr(x|z_\beta)\dd x \\
    &= \int \sum_{j\in\pa(k)} w_{jk} g_j(x)\pr(x|z_\beta) \dd x \label{eq:backward2}
\end{align}
Notice that we have:
\begin{equation}
    \pr(x|z_\beta) = \pr(x_j|z_\beta)\pr(x_{-j}|x_j,z_\beta) = \pr(x_j|z_\beta)\pr(x_{-j}|x_j)
\end{equation}
since $x_j$ covers $x_k$ which determines $z_\beta$. Therefore, for each item we have: 
\begin{align}
    \int g_j(x)\pr(x|z_\beta) \dd x &= \int g_j(x)\pr(x_j|z_\beta)\pr(x_{-j}|x_j) \dd x \\
    &=\int \left(\int g_j(x)\pr(x_{-j}|x_j)\dd x_{-j}\right)\pr(x_j|z_\beta)\dd x_j \\
    &= \int g_j(x_j)\pr(x_j|z_\beta)\dd x_j \label{eq:backward1}
\end{align}

Then we use the broadness of knowledge:
\begin{align}
    \pr(x_j|z_\beta) &= \sum_{z_\alpha}\pr(x_j, z_\alpha|z_\beta) \\
    &= \sum_{z_\alpha}\pr(x_j |z_\alpha, z_\beta)\pr(z_\alpha|z_\beta) \\
    &= \sum_{z_\alpha}\pr(x_j |z_\alpha)\pr(z_\alpha|z_\beta)
\end{align}

Following Eqn.~\ref{eq:backward1}, we now have:
\begin{align}
    \int g_j(x)\pr(x|z_\beta) \dd x &= 
    \int g_j(x_j) \sum_{z_\alpha}\pr(x_j |z_\alpha)\pr(z_\alpha|z_\beta) \dd x_j \\
    &=\sum_{z_\alpha} \left(\int g_j(x_j) \pr(x_j |z_\alpha) \dd x_j \right) \pr(z_\alpha|z_\beta) \\
    &=\sum_{z_\alpha} g_j(z_\alpha) \pr(z_\alpha|z_\beta) \\
    &= \ee2{z_\alpha|z_\beta}{g_j(z_\alpha)}
\end{align}

Putting it back to Eqn.~\ref{eq:backward2} and we have:
\begin{equation}
    g^{\mathrm{raw}}_k(z_\beta) = \sum_{j\in\pa(k)} w_{jk} \ee2{z_\alpha|z_\beta}{g_j(z_\alpha)}
\end{equation}

Using Eqn.~\ref{eq:grad-collect}: 
\begin{align}
g_k(z_\beta) &= \int g_k(x_k)\pr(x_k|z_\beta)\dd x_k \\ &= \int f_k'(x_k)g^{\mathrm{raw}}_k(x_k) \pr(x_k|z_\beta) \dd x_k \\
&= \ee2{X_k|z_\beta}{f'_k(X_k)g_k^{\mathrm{raw}}(X_k)}
\end{align}
The un-correlation between $g_k^{\mathrm{raw}}(\cdot)$ and $f'_k(\cdot)$ means that 
\begin{align}
\ee2{X_k|z_\beta}{f'_k g_k^{\mathrm{raw}}} &=  \ee2{X_k|z_\beta}{f'_k}\cdot\ee2{X_k|z_\beta}{g_k^{\mathrm{raw}}}
\end{align}
Similarly for $f_j(z_\alpha)$.
\end{proof}

\subsection{Exactness of reformulation}
\begin{theorem}
If $\pr(x_j|z_\alpha)$ is a delta function for all $\alpha$, then the conditions of Thm.~\ref{thm:coarse-model} hold and the reformulation becomes exact.
\end{theorem}
\begin{proof}
The fact that $\pr(x_j|z_\alpha)$ is a delta function means that there exists a function $\phi_j$ so that:
\begin{equation}
    \pr(x_j|z_\alpha) = \delta(x_j - \phi_j(z_\alpha))
\end{equation}
That is, $z_\alpha$ contains all information of $x_j$ (or $x_\alpha$). Therefore, 
\begin{itemize}
    \item \textbf{Broadness of knowledge.} $z_\alpha$ contains strictly more information than $z_\beta$ for $\beta\in \ch(\alpha)$, therefore $\pr(x_j|z_\alpha, z_\beta) = \pr(x_j|z_\alpha)$. 
    \item\textbf{Focus of knowledge.}
    $z_\beta$ captures all information of $z_k$, so $\pr(x_k|z_\alpha, z_\beta) = \pr(x_k|z_\beta)$.
    \item\textbf{Decorrelation.} For any $h_1(x_j)$ and $h_2(x_j)$ we have 
    \begin{align}
        \ee2{X_j|z_\alpha}{h_1h_2} &= \int  h_1(x_j)h_2(x_j)\pr(x_j|z_\alpha)\dd x_j \\
        &= \int h_1(x_j)h_2(x_j)\delta(x_j - \phi_j(z_\alpha))\dd x_j \\
        &=
        h_1(\phi_j(z_\alpha))h_2(\phi_j(z_\alpha)) \\
        &= \int h_1(x_j)\pr(x_j|z_\alpha)\dd x_j \int h_2(x_j)\pr(x_j|z_\alpha)\dd x_j \\ 
        &=
        \ee2{X_j|z_\alpha}{h_1}\ee2{X_j|z_\alpha}{h_2}
    \end{align}
\end{itemize}
\end{proof}

\subsection{Matrix form}
\begin{theorem}[Matrix Representation of Reformulation]
\begin{equation}
    F_\alpha\! =\!D_{\alpha} \!\circ\! \sum_{\beta\in \ch(\alpha)} P_{\alpha\beta} F_\beta W_{\beta\alpha} ,\quad
    \tilde G_\beta\! =\! D_\beta \!\circ\! \sum_{\alpha\in\pa(\beta)} P_{\alpha\beta}^T \tilde G_\alpha W_{\beta\alpha}^T  ,\quad \Delta W_{\beta\alpha}\! =\! (P_{\alpha\beta} F_\beta)^T \tilde G_\alpha
    \label{eqn:matrix-form-simplified}  
\end{equation}
\end{theorem}
\begin{proof}
We first consider one certain group $\alpha$ and $\beta$, which uses $x_\alpha $ and $x_\beta$ as the receptive field. For this pair, we can write Eqn.~\ref{eq:induction-appendix} in the following matrix form:
\begin{subequations}
\label{eqn:matrix-form-appendix}
\begin{align}
    F^{\raw}_\alpha &= P_{\alpha\beta} F_\beta W_{\beta\alpha} \\
    F_{\alpha} &= F^{\raw}_\alpha \circ D_{\alpha} \\
    G^\raw_\beta &= \left(P^b_{\alpha\beta}\right)^T G_{\beta} W_{\beta\alpha}^T \\    
    G_\beta &= G^\raw_\beta \circ D_{\beta}
\end{align}
\end{subequations}

\begin{table}
\centering
\begin{tabular}{|l||l|l|}
\hline
 & Dimension & Description \\
\hline\hline
$F_\alpha$, $D_\alpha$ & $m_\alpha$-by-$n_\alpha$ & Activation $f_j(z_\alpha)$ and gating prob $f'_j(z_\alpha)$ in group $\alpha$. \\
\hline
$G_\alpha$, $\tilde G_\alpha$ & $m_\alpha$-by-$n_\alpha$ & Gradient $g_j(z_\alpha)$ and unnormalized gradient $\tilde g_j(z_\alpha)$ in group $\alpha$. \\
\hline
$W_{\beta\alpha}$ & $n_\beta$-by-$n_\alpha$ & Weight matrix that links group $\beta$ and $\alpha$. \\
\hline
$P_{\alpha\beta}$, $P^b_{\alpha\beta}$ & $m_\alpha$-by-$m_\beta$ & Prob $\pr(z_\beta|z_\alpha)$, $\pr(z_\alpha|z_\beta)$ of events between group $\beta$ and $\alpha$.\\
\hline
$\Lambda_\alpha$ & $m_\alpha$-by-$m_\alpha$ & Diagonal matrix encoding prior prob $\pr(z_\alpha)$. \\
\hline
\end{tabular}
\caption{Matrix Notation. See Eqn.~\ref{eqn:matrix-form-simplified}.}
\label{tbl:matrix-notation}
\end{table}

Using $\Lambda_\beta (P^b_{\alpha\beta})^T = P_{\alpha\beta}^T\Lambda_\alpha$ and $\tilde G_\alpha = \Lambda_\alpha G_\alpha$, we could simplify Eqn.~\ref{eqn:matrix-form-appendix} as follows: 
\begin{subequations}
\label{eqn:matrix-form-simplified-pair-appendix}
\begin{align}
    F_\alpha &= P_{\alpha\beta} F_\beta W_{\beta\alpha} \circ D_{\alpha} \\
    \tilde G_\beta &= \left(P_{\alpha\beta}\right)^T \tilde G_{\beta} W_{\beta\alpha}^T \circ D_{\beta} 
\end{align}
\end{subequations}

Therefore, using the fact that $\sum_{j\in\pa(k)} = \sum_{\alpha\in\pa(\beta)} \sum_{j\in\alpha}$ (where $\beta = \rf{k}$) and $\sum_{k\in\ch(j)} = \sum_{\beta\in\ch(\alpha)} \sum_{k\in\beta}$ (where $\alpha = \rf{j}$), and group all nodes that share the receptive field together, we have:

\begin{subequations}
\label{eqn:matrix-form-simplified-appendix}
\begin{align}
    F_\alpha &= D_{\alpha} \circ \sum_{\beta\in \ch(\alpha)} P_{\alpha\beta} F_\beta W_{\beta\alpha} \\
    \tilde G_\beta &= D_\beta \circ \sum_{\alpha\in\pa(\beta)} P_{\alpha\beta}^T \tilde G_\alpha W_{\beta\alpha}^T
\end{align}
\end{subequations}
For the gradient update rule, from Eqn.~\ref{eq:x-weight-update} notice that:
\begin{align}
    \Delta w_{jk} &= \ee2{x}{f_k(x)g_j(x)} \\
    &= \int f_k(x)g_j(x)\pr(x)\dd x \\
    &= \int f_k(x)g_j(x)\sum_{z_\alpha} \pr(x|z_\alpha)\pr(z_\alpha)\dd x \\
    &= \sum_{z_\alpha} \int f_k(x)g_j(x) \pr(x|z_\alpha)\pr(z_\alpha)\dd x
\end{align}
We assume decorrelation so we have:
\begin{align}
    \Delta w_{jk} &= \sum_{z_\alpha} \ee2{X|z_\alpha}{f_k(x)} g_j(z_\alpha)\pr(z_\alpha)\\
    &=\sum_{z_\alpha} \ee2{X_k|z_\alpha}{f_k(x_k)} \tilde g_j(z_\alpha)
\end{align}
For $\ee2{X_k|z_\alpha}{f_k(x_k)}$, again we use focus of knowledge: 
\begin{align}
    \ee2{X_k|z_\alpha}{f_k(x_k)} &= \int f_k(x_k)\pr(x_k|z_\alpha)\dd x_k \\
    &= \sum_{z_\beta} \int f_k(x_k) \pr(x_k|z_\alpha, z_\beta)\pr(z_\beta|z_\alpha)\dd x_k \\
    &= \sum_{z_\beta} \int f_k(x_k) \pr(x_k|z_\beta)\pr(z_\beta|z_\alpha)\dd x_k \\
    &= \sum_{z_\beta}f_k(z_\beta) \pr(z_\beta|z_\alpha) 
\end{align}

Put them together and we have:
\begin{equation}
    \Delta w_{jk} = \sum_{z_\alpha}\sum_{z_\beta}f_k(z_\beta) \tilde g_j(z_\alpha)\pr(z_\beta|z_\alpha) = \ee2{z_\alpha, z_\beta}{f_k(z_\beta)g_j(z_\alpha)}
\end{equation}

Write it in concise matrix form and we get:
\begin{equation}
    \Delta W_{\beta\alpha} = (P_{\alpha\beta} F_\beta)^T \tilde G_\alpha
\end{equation}
\end{proof}

\subsection{Batch Norm as a projection}
\begin{theorem}[Backpropagation of Batch Norm]
For a top-down gradient $\vg$, BN layer gives the following gradient update ($P^\perp_{\bnvinput, \vone}$ is the orthogonal complementary projection of subspace $\{\bnvinput, \vone\}$):
\begin{equation}
    \vg_\vf = J^{BN}(\bnvinput)\vg = \frac{c_1}{\sigma}P^\perp_{\bnvinput, \vone}\vg, \quad \vg_\vc = S(\bnvinput{})^T \vg
    \label{eq:batch-norm-projection}
\end{equation}
\end{theorem}
\begin{proof}
We denote pre-batchnorm activations as $\bninput{i} = f_j(x_i)$ ($i = 1 \ldots N$). In Batch Norm, $\bninput{i}$ is whitened to be $\bnstandard{i}$, then linearly transformed to yield the output $\bnoutput{i}$:
\begin{equation}
    \bnzeromean{i} = \bninput{i} - \mu, \quad\bnstandard{i} = \bnzeromean{i} /\sigma, \quad\bnoutput{i} = c_1 \bnstandard{i} + c_0
\end{equation}
where $\mu = \frac{1}{N}\sum_i \bninput{i}$ and $\sigma^2 = \frac{1}{N}\sum_i (\bninput{i} - \mu)^2$ and $c_1$, $c_0$ are learnable parameters.

While in the original batch norm paper, the weight update rules are super complicated and unintuitive (listed here for a reference):
\begin{figure}[ht!]
    \centering
    \includegraphics[width=0.7\textwidth]{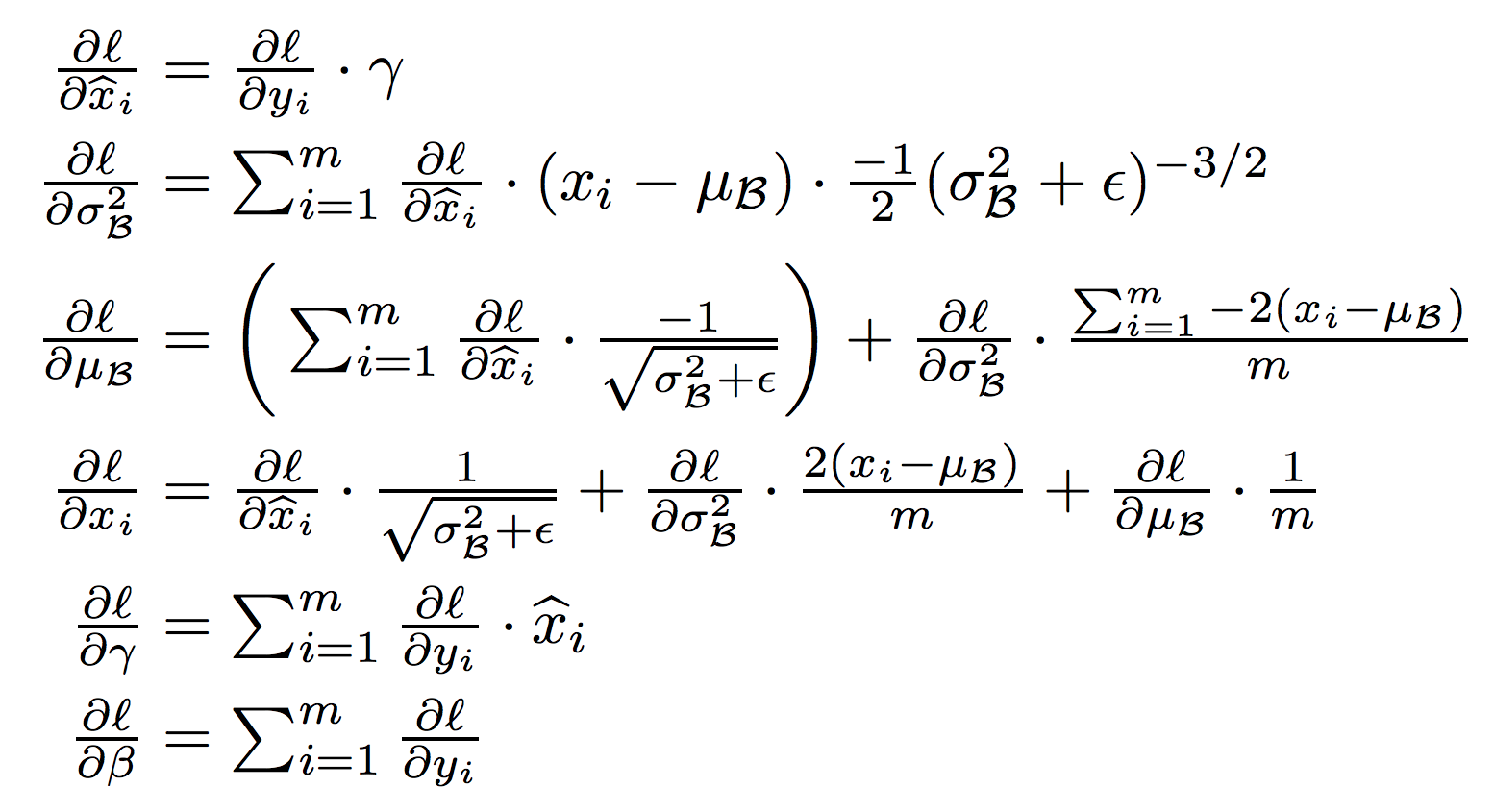}
    \caption{Original BN rule in~\citep{batchnorm}.}
\end{figure}

It turns out that with vector notation, the update equations have a compact vector form with clear geometric meaning. 

To achieve that, we first write down the vector form of forward pass of batch normalization:
\begin{equation}
    \bnvzeromean = P^\perp_1 \bnvinput, \quad \bnvstandard{} = \bnvzeromean / \|\bnvzeromean\|_{\mathrm{uni}}, \quad \bnvoutput = c_1 \bnvstandard + c_0\vone = S(\bnvinput)\vc
\end{equation}
where $\bnvinput$, $\bnvzeromean$, $\bnvstandard$ and $\bnvoutput$ are vectors of size $N$, $P^\perp_1 \equiv I - \frac{\vone\vone^T}{N}$ is the projection matrix that
centers the data, $\sigma = \|\vf\|_\mathrm{uni} = \frac{1}{\sqrt{N}}\|\vf\|_2$ and $\vc \equiv [c_1, c_0]^T$ are the parameters in Batch
Normalization and $S(\bnvinput)\equiv [\bnvstandard{}, \vone]$ is the standardized
data. Note that $S(\bnvinput{})^TS(\bnvinput{}) = N \cdot I_2$ ($I_2$ is $2$-by-$2$ identity matrix) and thus $S(\vx)$ is an column-orthogonal $N$-by-$2$ matrix. 
If we put everything together, then we have:
\begin{equation}
    BN(\vf) = c_1\frac{P^\perp_1 \vf}{\norm{P^\perp_1 \vf}_{\mathrm{uni}}} + c_0\vone
\end{equation}

\begin{figure}
    \centering
    \includegraphics[width=0.8\textwidth]{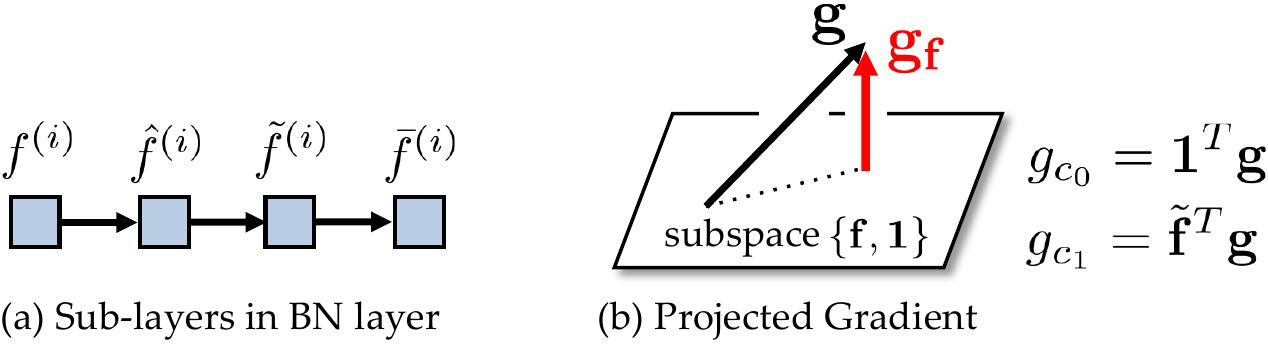}
    \caption{Analysis of Batch Normalization. \textbf{(a)} A Batch Normalization layer could be decomposed into three sublayers (zero-mean, unit-variance, affine). \textbf{(b)} The down
      gradient $\vg_\vf$ is a projection of input gradient $\vg$ onto the orthogonal complementary space spanned by $\{\vf, \vone\}$.}\label{fig:bn-appendix}
\end{figure}

Using this notation, we can compute the Jacobian of batch normalization layer. Specifically, for any vector $\vf$, we have:
\begin{equation}
\frac{\dd\left(\frac{\vf}{\norm{\vf}}\right)}{\dd\vf} = \frac{1}{\norm{\vf}}\left(I - \frac{\vf\vf^T}{\norm{\vf}^2}\right) = \frac{1}{\norm{\vf}}P^\perp_{\vf}
\end{equation}
where $P^\perp_{\vf}$ projects a vector into the \emph{orthogonal complementary} space of $\vf$. Therefore we have: 
\begin{equation}
    \jbn{\vf} = \frac{\dd\bnvoutput{}}{\dd\vf} = c_1\frac{\dd\bnvstandard{}}{\dd\vf} = 
    c_1\frac{\dd\bnvzeromean{}}{\dd\vf}\frac{\dd\bnvstandard{}}{\dd\bnvzeromean{}} = 
\frac{c_1}{\sigma}P^c_{\bnvstandard{}, \vone} \label{eqn:bn-jacobian}
\end{equation}
where $P^\perp_{\bnvstandard, \vone} = I - \frac{\s{\vf}\s{\vf}^T}{N}$ is a symmetric projection matrix that projects the input gradient to the orthogonal complement space spanned by $\tvx$ and $\vone$ (Fig.~\ref{fig:bn}(b)). Note that the space spanned by $\bnvstandard$ and $\vone$ is also the space spanned by $\vf$ and $\vone$, since $\bnvstandard = (\vf - \mu\vone) / \sigma$ can be represented linearly by $\vf$ and $\vone$. Therefore $P^\perp_{\bnvstandard, \vone} = P^\perp_{\vf, \vone}$. 

An interesting property is that since $BN(\vf)$ returns a vector in the subspace of $\vf$ and $\vone$, for the $N$-by-$N$ Jacobian matrix of Batch Normalization, we have:
\begin{equation}
    \jbn{\vf}BN(\vf) = \jbn{\vf}\vone = \jbn{\vf}\vf = \vzero
\end{equation} 
Following the backpropagation rule, we get the following gradient update for batch normalization. If $\vg = \partial L / \partial \bar\vf$ is the gradient from top, then
\begin{equation}
    \vg_{\vc} = \s{\vx}^T\vg, \quad\quad \vg_{\vf} = \jbn{\vf}\vg \label{eqn:bn-projective-gradient}
\end{equation}

Therefore, any gradient (vector of size $N$) that is back-propagated to the input of BN layer will be automatically orthogonal to that activation (which is also a vector of size $N$). 
\end{proof}

\subsubsection{Conserved Quantity in Batch Normalization for ReLU activation}
One can find an interesting quantity, by multiplying $g_j(x)$ on both side of the forward equation in Eqn.~\ref{eq:x-update} and taking expectation:
\begin{equation}
    \ee2{x}{g_jf_j} = \ee2{x}{\sum_{k\in \ch(j)} w_{jk} f_kg_j} = \sum_{k\in\ch(j)} w_{jk}\Delta w_{jk}
\end{equation}
Using the language of differential equation, we know that:
\begin{equation}
    \int_0^t \ee2{x}{g^{(t')}_jf^{(t')}_j} \dd t' = E_j(t) - E_j(0)
\end{equation}
where $E_j = \frac{1}{2}\sum_{k\in \ch(j)} w^2_{jk} = \frac{1}{2} \|W_{j\cdot}\|^2$. If we place Batch Normalization layer just after ReLU activation and linear layer, by BN property, since $\ee2{x}{g_jf_j} \equiv 0$ for all iterations, the row energy $E_j(t)$ of weight matrix $W$ of the linear layer is conserved over time. This might be part of the reason why BN helps stabilize the training. Otherwise energy might ``leak'' from one layer to nearby layers. 

\subsection{Nonlinear versus linear}
\begin{theorem}[Expressibility of ReLU Nonlinearity]
\label{thm:sufficient-node-appendix}
Assuming $m_\alpha = n_\alpha = \mathcal{O}(\exp(\mathrm{sz}(\alpha)))$, where $\sz(\alpha)$ is the size of receptive field of $\alpha$. If each $P_{\alpha\beta}$ is all-vert, then: ($\omega$ is top-level receptive field) 
\begin{equation}
    \min_W Loss_\mathrm{ReLU}(W) = 0, \quad \min_W Loss_\mathrm{Linear}(W) = \mathcal{O}(\exp(\sz(\omega)))
\end{equation}
Here we define $Loss(W) \equiv \|F_\omega - I\|^2_F$.
\end{theorem}
\begin{proof}
We prove that in the case of nonlinearity, there exists a weight so that the activation $F_\alpha = I$ for all $\alpha$. We prove by induction. The base case is trivial since we already know that $F_\alpha = I$ for all leaf regions.

Suppose $F_\beta = I$ for any $\beta\in\ch(\alpha)$. Since $P_{\alpha\beta}$ is all-vert, every row is a vertex of the convex hull, which means that for $i$-th row $p_i$, there exists a weight $\vw_i$ and $b_i$ so that $\vw_i^T p_i + b_i = 1/|\ch(\alpha)| > 0$ and $\vw_i^T p_j + b_i < 0$ for $j \neq i$. Put these weights and biases together into $W_{\beta\alpha}$ and we have
\begin{equation}
    F_\alpha^{\raw} = \sum_\beta P_{\alpha\beta} F_\beta W_{\beta\alpha} = \sum_\beta P_{\alpha\beta}W_{\beta\alpha}
\end{equation}
All diagonal elements of $F_\alpha^{\raw}$ are $1$ while all off-diagonal elements are negative. Therefore, after ReLU, $F_{\alpha} = I$. Applying induction, we get $F_{\omega} = I$ and $G_\omega = I - F_\omega = 0$. Therefore, $Loss_{\mathrm{ReLU}}(W) = \|G_\omega\|^2_F = 0$.

In the linear case, we know that $\rank(F_\alpha) \le \sum_\beta \rank(P_{\alpha\beta}F_\beta W_{\beta\alpha}) \le \sum_\beta \rank(F_\beta)$, which is on the order of the size $\sz(\alpha)$ of $\alpha$'s receptive field (Note that the constant relies on the overlap between receptive fields). However, at the top-level, $m_\omega=n_\omega = \mathcal{O}(\exp(\mathrm{sz}(\omega)))$, i.e., the information contained in $\alpha$ is exponential with respect to the size of the receptive field. By Eckart–Young–Mirsky theorem, we know that there is a lower bound for low-rank approximation. Therefore, the loss for linear network $Loss_\mathrm{linear}$ is at least on the order of $m_0$, i.e., $Loss_\mathrm{linear} = \mathcal{O}(m_\omega)$.
Note that this also works if we have BN layer in-between, since BN does a linear transform in the forward pass. 
\end{proof}

\begin{theorem}
The following two are correct:
\begin{itemize}
\item[(1)] If $F$ is row full rank, then $\vert(PF) = \vert(P)$. 
\item[(2)] $PF$ is all-vert iff $P$ is all-vert.
\end{itemize}
\end{theorem}

\begin{proof}
For (1), note that each row of $PF$ is $p_i^TF$. If $F$ is row full rank, then $F$ has pseudo-inverse $F'$ so that $FF' = I$. Therefore, if $p_i$ is not a vertex:
\begin{equation}
    p_i = \sum_{j\neq i} a_j p_j, \quad \sum_j a_j = 1, a_j \ge 0, 
\end{equation}
then $p_i^TF$ is also not a vertex and vice versa. Therefore, $\vert(PF) = \vert(P)$. (2) follows from (1).
\end{proof}

\subsection{disentangled Representation}
We first start with two lemmas. Both of them have simple proofs. 
\begin{lemma}
\label{lemma:disentangled}
Distribution representations have the following property:
\begin{itemize}
    \item[(1)] If $F^{(i)}_\alpha$ is disentangled, $F_\alpha = \sum_i w_i F^{(i)}_\alpha$ is also disentangled. 
    \item[(2)] If $F_\alpha$ is disentangled and $h$ is any per-column element-wise function, then $h(F_\alpha)$ is disentangled. 
    \item[(3)] If $F^{(i)}_\alpha$ are disentangled, $h_i$ are per-column element-wise function, then $h_1(F^{(1)}_\alpha) \circ h_2(F^{(2)}_\alpha) \ldots \circ h_n(F^{(n)}_\alpha)$ is disentangled.
\end{itemize}
\end{lemma}
\begin{proof}
(1) follows from properties of tensor product. For (2) and (3), note that the $j$-th column of $F_\alpha$ is  $F_{\alpha, :j} = \vone\otimes \ldots \vf_j \ldots \otimes\vone$, therefore $h^j(F_{\alpha, :j}) = \vone\otimes \ldots h^j(\vf_j) \ldots \otimes\vone$, and $h^j_1(F^{(1)}_{\alpha, :j}) \circ h^j_2(F^{(2)}_{\alpha, :j}) = \vone\otimes \ldots h^j_1(\vf^{(1)}_j) \circ h^j_2(\vf^{(2)}_j) \ldots \otimes\vone$.
\end{proof}

Given one child $\beta\in\ch(\alpha)$, denote
\begin{eqnarray}
P_{S_j} &=& P_{\bfactor{\alpha}{j}\bfactor{\beta}{S_j}} = \left[\pr(z_{\bfactor{\beta}{S_j}} | z_{\bfactor{\alpha}{j}}) \right]\\
\vw_{S_j} &=& W_{\beta\alpha}[S_j, j] \\
\vp_{\bfactor{\alpha}{j}} &=& \left[\pr(\bfactor{\alpha}{j} = 0), \pr(\bfactor{\alpha}{j} = 1)\right]^T
\end{eqnarray}
We have $P_{S_j}\vone = \vone$ and $\vone^T \vp_{\bfactor{\alpha}{j}} = 1$. Note here for simplicity, $\vone$ represents all-one vectors of any length, determined by the context. 

Since $F_\alpha$ and $G_\beta$ are disentangled, their $j$-th column can be written as:
\begin{eqnarray}
F_{\beta, :j} &=& \vone\otimes \ldots \otimes \vf_j \otimes \ldots \otimes\vone \\
\tilde G_{\alpha, :j} &=& \vp_{\bfactor{\alpha}{1}}\otimes \ldots \otimes\tilde\vg_j \otimes \ldots \otimes \vp_{\bfactor{\alpha}{n_\alpha}}
\end{eqnarray}

For simplicity, in the following proofs, we just show the case that $n_\alpha = 2$, $n_\beta = 3$, $z_\alpha = \left[z_{\bfactor{\alpha}{1}}, z_{\bfactor{\alpha}{2}}\right]$ and $S = \{S_1, S_2\} = \{\{1, 2\}, \{3\}\}$. We write $\vf_{1,2} = [\vf_1 \otimes \vone, \vone \otimes \vf_2]$ as a 2-column matrix. The general case is similar and we omit here for brevity. 

\begin{theorem}[disentangled Forward]
If for each $\beta\in\ch(\alpha)$,  $P_{\alpha\beta}$ can be written as a tensor product $P_{\alpha\beta} = \bigotimes_i P_{\bfactor{\alpha}{i}\bfactor{\beta}{S^{\alpha\beta}_i}}$ where $\{S^{\alpha\beta}_i\}$ are $\alpha\beta$-dependent disjointed set, $W_{\beta\alpha}$ is separable with respect to $\{S^{\alpha\beta}_i\}$, $F_\beta$ is disentangled, then $F_\alpha$ is also disentangled (with/without ReLU /Batch Norm).
\end{theorem}
\begin{proof}
For a certain $\beta\in\ch(\alpha)$, we first compute the quantity $P_{\alpha\beta}F_\beta$:
\begin{equation}
P_{\alpha\beta}F_\beta = \left(P_{1,2} \otimes P_{3}\right)\left[\vf_{1,2} \otimes \vone,\ \ \vone \otimes \vf_3\right] = \left[P_{1,2}\vf_{1,2} \otimes \vone, \ \ 
\vone \otimes P_{3}\vf_3\right] \label{eq:pf}
\end{equation}

Therefore, the forward information sent from $\beta$ to $\alpha$ is:
\begin{eqnarray}
F^{\raw}_{\beta\rightarrow\alpha} &=& P_{\alpha\beta}F_\beta W_{\beta\alpha} = \left[P_{1,2}\vf_{1,2} \otimes \vone,\ \  
\vone \otimes P_{3}\vf_3\right]
\left[
\begin{array}{cc}
     \vw_{1,2} & 0 \\
     0 & \vw_{3} 
\end{array} 
\right] \\
&=& \left[P_{1,2}\vf_{1,2}\vw_{1,2} \otimes \vone, \ \ 
\vone \otimes P_{3}\vf_3\vw_3\right] 
\label{eq:forward-alpha}
\end{eqnarray}
Note that both $P_{1,2}\vf_{1,2}\vw_{1,2}$ and $P_{3}\vf_3\vw_3$ are 2-by-1 vectors. Therefore, for each $\beta\in\ch(\alpha)$, $F^{\raw}_{\beta\rightarrow\alpha}$ is disentangled. By Lemma~\ref{lemma:disentangled}, both $F^{\raw}_\alpha = \sum_{\beta\in\ch(\alpha)}F^{\raw}_{\beta\rightarrow\alpha}$ and the nonlinear response $F_\alpha$ are disentangled. By Eqn.~\ref{eq:batch-norm-projection-z-alpha}, the forward pass of Batch Norm is a per-column element-wise function, so BN also preserves disentangledness.  
\end{proof}

\begin{theorem}[Separable Weight Update]
If $P_{\alpha\beta} = \bigotimes_i P_{\bfactor{\alpha}{i}\bfactor{\beta}{S_i}}$, both $F_\beta$ and $\tilde G_\alpha$ are disentangled, $\vone^T \tilde G_\alpha = \vzero$, then the gradient update $\Delta W_{\beta\alpha}$ is separable with respect to $\{S_i\}$.
\end{theorem}
\begin{proof}
Following Eqn.~\ref{eqn:matrix-form-simplified-appendix} and Eqn.~\ref{eq:pf}, we have: 
\begin{eqnarray}
\Delta W_{\beta\alpha} &=&
\left(P_{\alpha\beta}F_\beta\right)^T \tilde G_\alpha \\
&=& \left[
\begin{array}{c}
(P_{1,2}\vf_{1,2})^T \otimes \vone^T \\
\vone^T \otimes (P_{3}\vf_3)^T 
\end{array}
\right]\left[
\tilde \vg_1 \otimes \vp_{\bfactor{\alpha}{2}},\ \  \vp_{\bfactor{\alpha}{1}} \otimes \tilde \vg_2 \right] \\
&=& \left[
\begin{array}{cc}
(P_{1,2}\vf_{1,2})^T \tilde\vg_1 & (P_{1,2}\vf_{1,2})^T\vp_{\bfactor{\alpha}{1}} \otimes \vone^T\tilde \vg_2 \\
\vone^T\tilde \vg_1 \otimes (P_3\vf_3)^T\vp_{\bfactor{\alpha}{2}} & (P_3\vf_3)^T \tilde\vg_2 \\
\end{array}
\right]
\end{eqnarray}
Since $\vone^T \tilde G_\alpha = \vzero$,  we have for any $j$, $\vone^T \tilde G_{\alpha,:j} = 0$ and thus $\vone^T\tilde \vg_j = 0$. Therefore, 
\begin{equation}
\Delta W_{\beta\alpha} = \mathrm{diag}\left((P_{1,2}\vf_{1,2})^T \tilde\vg_1,\ \ (P_3\vf_3)^T \tilde\vg_2\right)
\end{equation}
which is separable with respect to $S$. In particular: 
\begin{equation}
    \Delta\vw_{1,2} = (P_{1,2}\vf_{1,2})^T \tilde\vg_1, \quad \Delta\vw_3 = (P_3\vf_3)^T \tilde\vg_2 
\end{equation}
\end{proof}

\subsubsection{Discussion about backpropagation of disentangled gradient} 
One problem remains. If $\{\tilde G_\alpha\}_{\alpha\in\pa(\beta)}$ are all disentangled, whether $\tilde G_\beta$ is disentangled? We can try computing the following quality:
\begin{eqnarray}
    \tilde G^\raw_{\alpha\rightarrow \beta} &=& P_{\alpha\beta}^T\tilde G_\alpha W_{\beta\alpha}^T \\
    &=& \left(P^T_{1,2} \otimes P^T_{3}\right)\left[\tilde \vg_1 \otimes \vp_{\bfactor{\alpha}{2}},\ \  \vp_{\bfactor{\alpha}{1}} \otimes \tilde \vg_2 \right]\left[
\begin{array}{cc}
     \vw^T_{1,2} & 0 \\
     0 & \vw^T_{3} 
\end{array}
\right] \\
    &=& \left[P_{1,2}^T\tilde\vg_1\otimes \vp_{\bfactor{\beta}{3}},\ \  \vp_{\bfactor{\beta}{1,2}} \otimes P_3^T\tilde\vg_2 \right]
    \left[
    \begin{array}{cc}
     \vw^T_{1,2} & 0 \\
     0 & \vw^T_{3} 
    \end{array}
    \right] \\
    &=& \left[
    P_{1,2}^T\tilde\vg_1 \vw^T_{1,2}\otimes \vp_{\bfactor{\beta}{3}}, \ \  
    \vp_{\bfactor{\beta}{1,2}}\otimes P_3^T\tilde\vg_2 \vw^T_3 
    \right] 
\end{eqnarray}
Note that here we use the following equality from total probability rule:
\begin{equation}
    P_3^T\vp_{\bfactor{\alpha}{2}} = \vp_{\bfactor{\beta}{3}}, \quad P_{1,2}^T\vp_{\bfactor{\alpha}{1}} = \vp_{\bfactor{\beta}{1,2}} \label{eq:gradient-beta}
\end{equation}
where $\vp_{\bfactor{\beta}{1,2}}$ is a 4-by-1 vector:
\begin{equation}
    \vp_{\bfactor{\beta}{1,2}} = \left[
    \begin{array}{c}
    \pr(z_{\bfactor{\beta}{1}} = 0, z_{\bfactor{\beta}{2}} = 0) \\
    \pr(z_{\bfactor{\beta}{1}} = 0, z_{\bfactor{\beta}{2}} = 1) \\ 
    \pr(z_{\bfactor{\beta}{1}} = 1, z_{\bfactor{\beta}{2}} = 0) \\
    \pr(z_{\bfactor{\beta}{1}} = 1, z_{\bfactor{\beta}{2}} = 1)
    \end{array}
    \right]
\end{equation}
Note that the ordering of these joint probability corresponds to the column order of $P_{1,2}$. 

Now with this example, we see that the backward case ( Eqn.~\ref{eq:gradient-beta}) is very different from the forward case (Eqn.~\ref{eq:forward-alpha}), in which  $\tilde G^\raw_{\alpha\rightarrow\beta}$ is no longer disentangled. Indeed, $P_{1,2}^T\tilde\vg_1 \vw^T_{1,2}$ is a 2-column matrix and $\vp_{\bfactor{\beta}{1,2}}$ is not a rank-1 tensor anymore. Intuitively this makes sense, if two low-level attributes have very similar behaviors, there is no way to distinguish the two via backpropagation.  

Note that we also \emph{cannot} assume independence: $\vp_{\bfactor{\beta}{1,2}} = \vp_{\bfactor{\beta}{1}} \otimes \vp_{\bfactor{\beta}{2}}$ since the independence property is in general not carried from layer to layer.

For general cases, $\tilde G^\raw_{\alpha\rightarrow\beta}$ takes the following form:
\begin{equation}
\tilde G^\raw_{\alpha\rightarrow\beta} = \left[P_{S^{\alpha\beta}_1}^T\tilde\vg^\alpha_1 \vw_{S^{\alpha\beta}_1}^T \otimes \bigotimes_{j=2}^{n_\alpha} \vp_{\bfactor{\beta}{S^{\alpha\beta}_j}},\quad \vp_{\bfactor{\beta}{S^{\alpha\beta}_1}}\otimes P_{S^{\alpha\beta}_2}^T\tilde\vg^\alpha_2 \vw_{S^{\alpha\beta}_2}^T \otimes \bigotimes_{j=3}^{n_\alpha} \vp_{\bfactor{\beta}{S^{\alpha\beta}_j}}, \quad \ldots \right]
\end{equation}

One hope here is that if we consider $\sum_{\alpha\in \pa(\beta)} \tilde G^\raw_{\alpha\rightarrow\beta}$, the summation over parent $\alpha$ could lead to a better structure, even for individual $\alpha$, $P_{S_1}^T\tilde\vg^\alpha_1 \vw_{S_1}^T$ is not 1-order tensor. For example, if $S^{\alpha\beta}_j = S_j$, then for the first column in $S_1$, due to $\vone^T\tilde\vg_j = 0$, we know that:
\begin{equation}
    \sum_{\alpha\in\pa(\beta)} P_{\alpha,S_1}^T\tilde\vg^\alpha_1\vw_{\alpha, S_1[1]}^T = \sum_{\alpha\in\pa(\beta)} c_\alpha
    (\vv^+_{\alpha,S_1} - \vv^-_{\alpha,S_1}) \label{eq:part}
\end{equation}
where $\vv^+_{\alpha,S_1} =     \pr(z_{\bfactor{\beta}{S_1}}|z_{\bfactor{\alpha}{1}} = 1)$ and $\vv^-_{\alpha,S_1} = \pr(z_{\bfactor{\beta}{S_1}}|z_{\bfactor{\alpha}{1}} = 0)$ and $P_{\alpha,S_i} = \left[ \begin{array}{c} 
\vv^-_{\alpha,S_1} \\
\vv^+_{\alpha,S_1}
\end{array}\right]$. If each $\alpha\in\pa(\beta)$ is informative in a diverse way, and $|S_1|$ is relatively small (e.g., 4), then $\vv^+_{\alpha,S_1} -
\vv^-_{\alpha,S_1} \neq \vzero$ and spans the probability space of dimension $2^{|S_1|} - 1$. Then we can always find $c_\alpha$ (or equivalently, weights) so that Eqn.~\ref{eq:part} becomes rank-1 tensor (or disentangled). Besides, the gating $D_\beta$, which is disentangled as it is an element-wise function of $F_\beta$, will also play a role in regularizing $\tilde G_\beta$.

We will leave this part to future work. 

\end{document}